\documentclass{article}
\usepackage[backref=page]{hyperref}       

\usepackage{mathtools, cuted}
\usepackage{bbm,amsthm,amsmath,amssymb}

\newtheorem{theo}{Theorem}[section]

\newtheorem{cor}[theo]{Corollary}
\newtheorem{lem}[theo]{Lemma}

\newtheorem{rem}[theo]{Remark}

\newtheorem{assump}{Assumption}

\usepackage{epsfig}
\usepackage{graphicx}
\usepackage{epstopdf}
\usepackage{comment}
\usepackage{array}
\usepackage{algorithm}
\usepackage{lscape}
\usepackage{algpseudocode}
\usepackage{setspace}
\usepackage{multicol}
\usepackage{multirow}
\usepackage{color}
\usepackage{colortbl}
\usepackage{xcolor}

\usepackage{placeins}
\usepackage[a4paper, total={7in, 8in}]{geometry}
\usepackage[%
    font={small,sf},
    labelfont=bf,
    format=hang,    
    format=plain,
    margin=0pt,
    width=1.2\textwidth,
]{caption}
\usepackage[list=true]{subcaption}

\begin{document}
\title{A finite sample analysis of generalisation for  minimum norm interpolators in the ridge function model with random design }
%
%
\author{Emmanuel Caron
\thanks{Laboratoire de Math\'ematiques d'Avignon (LMA), 
Université d’Avignon
Campus Jean-Henri Fabre
301, rue Baruch de Spinoza,
BP 21239
F-84 916 AVIGNON Cedex 9. Most of this work was done while the author was with the ERIC Laboratory, University of Lyon 2, Lyon, France} \and Stéphane Chrétien\thanks{ERIC Laboratory and UFR ASSP, University of Lyon 2, 
5 avenue Mendes France, 69676 Bron Cedex. He is invited researcher at the Mathematical Institute, University of Oxford, Oxford OX2 6GG, UK, the National Physical Laboratory, 1 Hampton Road, Teddington, TW11 0LW, UK, The Alan Turing Institute, British Library, 96 Euston Road, London NW1 2DB, UK, and FEMTO-ST Institute/Optics Department, 15B Avenue des Montboucons, F-25030 Besan{\c c}on, France.}}

\maketitle

\begin{abstract} 
Recent extensive numerical experiments in high scale machine learning have allowed to uncover a quite counterintuitive phase transition, as a function of the ratio between the sample size and the number of parameters in the model. As the number of parameters $p$ approaches the sample size $n$, the generalisation error increases, but surprisingly, it starts decreasing again past the threshold $p=n$. This phenomenon, brought to the theoretical community attention in \cite{belkin2019reconciling}, has been thoroughly investigated lately, more specifically for simpler models than deep neural networks, such as the linear model when the parameter is taken to be the minimum norm solution to the least-squares problem, firstly in the asymptotic regime when $p$ and $n$ tend to infinity, see e.g. \cite{hastie2019surprises}, and recently in the finite dimensional regime and more specifically for linear models \cite{bartlett2020benign}, \cite{tsigler2020benign}, \cite{lecue2022geometrical}. In the present paper, we propose a finite sample analysis of non-linear models of \textit{ridge} type, where we investigate the \textit{overparametrised regime} of the double descent phenomenon for both the \textit{estimation problem} and the \textit{prediction} problem. Our results provide a precise analysis of the distance of the best estimator from the true parameter as well as a generalisation bound which complements recent works of \cite{bartlett2020benign} and \cite{chinot2020benign}. Our analysis is based on tools closely related to the continuous Newton method \cite{neuberger2007continuous} and a refined quantitative analysis of the performance in prediction of the minimum $\ell_2$-norm solution.
\end{abstract}

\section{Introduction}

The tremendous achievements in  deep learning theory and practice 
have received great attention in the applied Computer Science, Artificial Intelligence and Statistics communities in the recent years. Many success stories related to the use of Deep Neural Networks have even been reported in the media and most data scientists are proficient with the Deep Learning tools available via opensource machine learning libraries such as Tensorflow, Keras, Pytorch and many others. 

One of the key ingredient in their success is the huge number of parameters involved in all current architectures. While enabling unprecedented expressivity capabilities, such regimes of overparametrisations appear very counterintuitive through the lens of traditional statistical wisdom. Indeed, as intuition suggests, overparametrisation often results in interpolation, i.e. zero training error. The expected outcome of such endeavors should result in very poor generalisation performance. Nevertheless interpolating deep neural networks still generalise well despite achieving zero or very small training error. 

Belkin et al. \cite{belkin2019reconciling} recently addressed the problem of resolving this paradox, and brought some new light on the complex and perplexing relationships between interpolation and generalization. In the linear model setting, in the regime where the weights are the minimum norm solution to the least-squares problem, overfitting was proved to be benign under strong design assumptions (i.i.d.) random matrices, when $p$ and $n$ tend to infinity, see e.g. \cite{hastie2019surprises}. More precise results were recently obtained in the finite dimensional regime and more specifically for linear models \cite{bartlett2020benign}, \cite{tsigler2020benign}, \cite{lecue2022geometrical}. Non-linear settings have also been considered such as in the particular instance of kernel ridge regression as studied in \cite{belkin2018understand}, \cite{mei2022generalization}, \cite{liang2018just} proved that interpolation can coexist with good generalization. More recent results in the nonlinear direction are, e.g.  \cite{frei2022benign} for shallow neural networks using a very general framework that combines the properties of gradient descent for a binary classification problem, but somehow restrictive assumptions on the dimension of the input of the form $n\|\mu\|^4 \gg p \geq C \max \left\{n\|\mu\|^2, n^2 \log (n / \delta)\right\}$. 
In this paper, we consider a statistical model where $\phi(X;\theta)=f(X_i^\top \theta)$, i.e. a model of the form 
\begin{align}
    \mathbb E[Y_i\mid X_i] & = f(X_i^\top\theta^{\ast} ), \qquad i=1,\ldots, n,
    \label{model}
\end{align} 
where $\theta^* \in \mathbb{R}^{p}$ and the function $f$ is assumed increasing. This setting is also known as the single index model and rigde function model. When the estimation of  $\theta^*$ is performed using Empirical Risk Estimation, i.e. by solving 
\begin{align}
\hat{\theta} & = \mathrm{argmin}_{\theta \in \Theta} \ \frac{1}{n} \sum_{i=1}^{n} \ell(Y_{i}- f(X_{i}^\top\theta))
\label{erm}
\end{align}
for a given smooth loss function $\ell$, we derive an upper bound on the prediction error that gives precise order of dependencies with respect to all the intrinsic parameters of the model, such as the dimensions $n$ and $p$, as well as various bounds on the derivatives of $f$ and of the loss function used in the Empirical Risk Estimation. 

Our contribution improves on the literature on the non-asymptotic regime of benign overfitting for non-linear models by providing concentration results on generalisation instead of controlling the expected risk only. More precisely, our results quantify the proximity in $\ell_2$ of a certain solution $\hat \theta$ of \eqref{erm} to $\theta^*$. Our proof of this first result utilises an elegant continuous Newton method argument initially promoted by Neuberger \cite{castro2001inverse}, \cite{neuberger2007continuous}. From this, we obtain a quantitative analysis of the performance in prediction of the minimum $\ell_2$-norm solution based on a careful study relying on high dimensional probability.


%

\section{Distance of the true model to the set of empirical risk minimisers}

We study a non-isotropic model for the rows $X_i^\top$ of the design matrix, and how the Cholesky factorisation of the covariance matrix can be leveraged to improve the prediction results. Let us now describe our mathematical model.

\subsection{Statistical model}

In this section, we make general assumptions on the feature vectors.

\begin{assump}
We assume that \eqref{model} holds and that $f$ and the loss function $\ell$ : $\mathbb{R} \mapsto \mathbb{R}$ satisfy the following properties
\begin{itemize}
    \item $f$ is strictly monotonic,
    \item $f'(z) \leq C_{f'}$ for some positive constant $C_{f'}$,
    \item $\ell'(0)=0$ and $\ell''$ is upper bounded by a constant $C_{\ell''}>0$.
\end{itemize}
\end{assump}

Concerning the statistical data, we will make the following set of assumptions

\begin{assump}
\label{assump_aniso_X}
We will require that 
\begin{itemize}
    \item the random column vectors $X_1,\ldots,X_n$ with values in $\mathbb R^p$ are independent random vectors, which can be written as $AZ_i$ for some matrix $A$ in $\mathbb R^{p\times p}$, with at least $n$ non zero singular values, and some independent  subGaussian vectors $Z_1,\ldots,Z_n$ in $\mathbb{R}^{p}$, with $\phi_2$-norm upper bounded by $K_Z$. The vectors $Z_i$ are such that the matrix
    \begin{align*}
    Z^\top & = 
    \begin{bmatrix}
        Z_1,
        \ldots,
        Z_n 
    \end{bmatrix}
\end{align*}
is full rank with probability one. We define $X=ZA^\top$; 
    \item for all $i=1,\ldots,n$, the random vectors $Z_{i}$ are assumed 
    \begin{itemize}
        \item to have a second moment matrix equal to the identity, i.e. $\mathbb{E} [Z_{i} Z_i^\top] = I_{p}$  
        \item to have $\ell_2$-norm equal\footnote{notice that this is different from the usual regression model, where the columns are assumed to be normalised.}  to $\sqrt{p}$.
    \end{itemize}

\end{itemize}  

\end{assump}

\begin{assump}
    The errors $\varepsilon_{i} = Y_{i} - \mathbb E[Y_{i}\mid X_i]$ are assumed to be independent subGaussian centered random variables with $\psi_2$-norm upper bounded by $K_\varepsilon$. 
\end{assump}

The performance of the estimators is measured by the theoretical risk $R : \Theta \mapsto \mathbb{R}$ by
\[R(\theta) = \mathbb{E} \Big[ \ell(Y- f(X^\top\theta)) \Big].\]

In order to estimate $\theta^*$, the Empirical Risk Minimizer $\hat{\theta}$ is defined as a solution to the following optimisation problem
\begin{align}
\hat{\theta} & \in \mathrm{argmin}_{\theta \in \Theta} \ \hat{R}_{n}(\theta), 
\label{optim_problem}
\end{align}
with 
\begin{align}
    \hat R_n(\theta) 
& = \frac{1}{n} \sum_{i=1}^{n} \ell(Y_{i}- f(X_{i}^\top\theta)).
\end{align}

Let us make an important assumption for what is to follow
\begin{assump}
Assume that the loss function $\ell$ and the function $f$ are such that
\begin{align}
     \vert\ell''(Y_i-f(X_i^\top \theta)) \ f'(X_i^\top \theta)^{2} - \ell'(Y_i-f(X_i^\top \theta)) \ f''(X_i^\top \theta)\vert & \ge \delta >0
    \label{lfcond}
\end{align}
for all $\theta$ in $\mathcal B_2(\theta^*,r)$, which is the $\ell_2$-ball of radius $r$ centered at $\theta^*$.
\label{assump_lfcond_aniso}
\end{assump}

Let us check Assumption \ref{assump_lfcond_aniso} in various simple cases. 
\begin{itemize}
    \item In the case of linear regression, we have $\ell(z)=\frac12 z^2$,  $\ell'(z)=z$,  $\ell''(z)=1$
 and 
 $f(z)=z$, $f'(z)=1$, $f''(z)=0$. Thus, we get  $\delta = 1$. 

\item In the case of the Huber loss and the Softplus function, we have \newline
$\ell(z)= \begin{cases}\frac{1}{2} z^2 & \text { for }|z| \leq \eta \\ \eta \cdot\left(|z|-\frac{1}{2} \eta\right), & \text { otherwise }\end{cases}$, $\quad$ $\ell'(z)= \begin{cases}
    z & \text { for }|z| \leq \eta \\ 
        \eta  & \text{ if } z > \eta \\
        - \eta & \text{ if } z <-\eta,
    \end{cases}$, $\quad$ $\ell''(z)= \begin{cases}
        1 & \text { for }|z| \leq \eta \\ 
        0  & \text{ if } \vert z\vert  > \eta
    \end{cases}$ \newline
and $f(z)=1/(1+\exp(-z))$, $f'(z)=\exp(-z)/(1+\exp(-z))^2$ and $f''(z)=\frac{2 e^{-2 z}}{\left(e^{-z}+1\right)^3}-\frac{e^{-z}}{\left(e^{-z}+1\right)^2}$.\\
\noindent Moreover, $f'(z)/f''(z)>1$. Then, when $\eta <1$, we can take $\delta =1-\eta$ as long as $\vert Y_i-f(X_i^\top \theta)\vert \le \eta$ for all $i=1,\ldots,n$ and all $\theta \in \mathcal B_2(\theta^*,r)$.
\end{itemize}

\textbf{Notation :} We denote by $s_i(A)$ the $i$-th singular value of the matrix $A$, in decreasing order, i.e. $s_1(A) \geq s_2(A) \geq \ldots \geq s_n(A)$. When $s_n(A)$ is non-zero, we denote by $\kappa_n(A) = \frac{s_1(A)}{s_n(A)}$ the n-th condition number of $A$.

\subsection{Distance to empirical risk minimisers via Neuberger's theorem}

Our main result in this section is the following Theorem. Their proofs are given in Section \ref{mainproof_under_aniso} and Section \ref{mainproof_over_aniso} in the appendix.  

\begin{theo}(Underparametrised setting)
\label{mainth_under_aniso}

Let $C_{K_{Z}}$ and $c_{K_{Z}}$ be positive constants depending only on the subGaussian norm $K_{Z}$. Let $\alpha$ be a real in $]0,1[$.
Assume that $p$ and $n$ are such that $C_{K_Z}^2 p < (1-\alpha)^2 n$ and that Assumption~\ref{assump_lfcond_aniso} is verified. Let 
\begin{align}
\label{requnder_aniso}
    r & = \frac{C_{(\ell'',f',\varepsilon)} \sqrt{p}}{\delta \ \left((1-\alpha) \sqrt{n} - C_{K_{Z}} \sqrt{p} \right)},
\end{align}
where $C_{(\ell'',f',\varepsilon)} = 6 \sqrt{C} C_{\ell''} C_{f'} K_{\varepsilon}$, with $C$ a positive absolute constant.

Then, with probability at least $1 -2 \exp{(-c_{K_{Z}} \alpha^{2} n)} -  \exp{( -p/2 )}$,
the unique solution $\hat{\theta}$ to the optimisation problem~\eqref{optim_problem} satisfies
\begin{align*}
    \Vert A^\top (\hat \theta -\theta^*)\Vert_2 & \le r,
\end{align*} 
where $\Vert \cdot \Vert_2$ is the usual Euclidean norm.

\end{theo}

\begin{theo}
\label{mainth_over_aniso}(Overparametrised setting)

Assume that $p$ and $n$ are such that $C_{K_Z}^2 \ n < (1-\alpha)^2 p$ and that Assumption~\ref{assump_lfcond_aniso} is verified.
Let 
\begin{align}
\label{reqover_aniso}
    r & =  \frac{C_{(\ell'',f',\varepsilon)} \sqrt{n}}{\delta\ \left( (1-\alpha) \sqrt{p} - C_{K_{Z}} \sqrt{n}\right)}.
\end{align}
Then, there exists a first order stationary point $\hat{\theta}$ to the optimisation problem~\eqref{optim_problem} such that, with probability larger than or equal to $1 -2 \exp{(-c_{K_{Z}} \alpha^{2} p)}  - \exp{\Big( -n/2 \Big)}$,
we have 
\begin{align}
     \Vert A^\top (\hat \theta -\theta^*) \Vert_{2} & \leq r.
\end{align}
\end{theo}

\begin{rem} 
Assume that $A=I$ (hence $X=Z$) for simplicity. In order to illustrate the previous results, consider the case where the noise level is of the same order as $\Vert X_i\Vert_2$, i.e.  $K_\varepsilon \sim \sqrt{p}$, and $\delta$ is independent of $r$, as in the linear case. In the underparametrised case, this implies that the error bound on $\Vert \hat \theta-\theta^*\Vert_2$  grows like $p/\sqrt{n}$. In the overparametrised case, the error bound given by Theorem \ref{mainth_over_aniso} is of order $\sqrt{n}$. This is potentially much smaller than $\sqrt{p}$ which is the natural order for  $\Vert \theta^*\Vert$ in the absence of sparsity. 

Note also that in the different setup of Candès and Plan \cite{candes2009near}, the rows of $X$ would have norm of the order of$\sqrt{p/n}$ and the noise would have subgaussian constant $K_\varepsilon$ independent of the dimensions $n$ and $p$. Multiplying by $\sqrt{n}$, we get a norm of the $X_i$'s of the order $\sqrt{p}$ and a subgaussian constant $K_\varepsilon$ of the order $\sqrt{n}$ for the noise. With this parametrisation of $K_\varepsilon$ in mind, in the underparametrised case, the bound is of the order of $\sqrt{p}$ and in the overparametrised case, the error bound given by Theorem \ref{mainth_over_aniso} would be of the order $n/\sqrt{p}$.
\label{rem_main}
\end{rem}

\subsection{The case of the linear models} 

\subsubsection{The case of linear regression}

Let us apply these theorems on the well-known linear regression model. In the linear case $f(X_i^\top h) = X_i^\top h$, the loss $\ell(z)=\frac12 z^2$ is quadratic and the optimisation problem~\eqref{optim_problem} is
\[\hat{\theta} = \mathrm{argmin}_{\theta \in \mathbb{R}^{p}} \ \frac{1}{2n} \sum_{i=1}^{n} (Y_{i} - X_i^\top \theta )^{2}.\]

We therefore have:  
\begin{cor}[Underparametrised case: linear model]
\label{maincor_under_aniso}

Let
\begin{align*}
    r & = \frac{ C_{\varepsilon} \sqrt{p}}{ \left( (1-\alpha) \sqrt{n} - C_{K_{Z}} \sqrt{p}\right)},
\end{align*}
where $C_{\varepsilon} = 6 \sqrt{C} K_{\varepsilon}$.
Under the same assumptions as Theorem~\ref{mainth_under_aniso}, the unique solution $\hat{\theta}$ to the optimisation problem~\eqref{optim_problem} satisfies
\begin{align*}
    \Vert A^\top (\hat{\theta} -\theta^*) \Vert_2 & \le r, 
\end{align*} 
with probability larger than or equal to $1 - 2 \exp{(-c_{K_{Z}} \alpha^{2} n)} - \exp{( -p/2)}$.
\end{cor}

\begin{proof}
Apply Theorem \ref{mainth_under_aniso} with $f(X_i^\top h) = X_i^\top h$ and $\ell(y) = \frac{1}{2} y^{2}$. Then we have
\begin{itemize}
\item[$\bullet$] $\ell'(y) = y$ and $\ell''(y) = 1$, hence $C_{\ell^{(k)}} = 0$, $k\ge 3$. 
\item[$\bullet$] $f(X_i^\top h) = X_i^\top h$, hence $C_{f'} = 1$ and $C_{f^{(k)}} = 0$ for $k \geq 2$.
\end{itemize}
Consequently, we have $\delta = 1$ in the linear case and Assumption~\ref{assump_lfcond_aniso} is verified. Furthermore, the constant $C_{(\ell'',f',\varepsilon)}$ simplifies to $6 \sqrt{C} K_{\varepsilon}$, which is denoted by $C_{\varepsilon}$.

This concludes the proof.
\end{proof}

\begin{cor}[Overparametrised case: linear model]
\label{maincor_over_aniso}

Let
\begin{align*}
    r & =  \frac{C_{\varepsilon} \sqrt{n}}{ \left((1-\alpha) \sqrt{p} - C_{K_{Z}} \sqrt{n}\right)},
\end{align*}
where $C_{\varepsilon} = 6 \sqrt{C} K_{\varepsilon}$.
Under the same assumptions as Theorem~\ref{mainth_over_aniso}, there exists a solution $\hat{\theta}$ to the optimisation problem~\eqref{optim_problem} such that, with probability larger than or equal to $1 - 2 \exp{(-c_{K_{Z}} \alpha^{2} p)} - \exp{( -n/2)}$,
we have 
\begin{align*}
     \Vert A^\top (\hat{\theta} - \theta^*) \Vert_{2} & \leq r.
\end{align*}
\end{cor}

\subsection{Discussion of the results}
Our results are based on a new zero finding approach inspired from \cite{neuberger2007continuous} and we obtain precise quantitative results in the finite sample setting for linear and non-linear models. Following the initial discovery of the "double descent phenomenon" in \cite{belkin2019reconciling}, many authors have addressed the question of precisely characterising the error decay as a function of the number of parameters in the linear and non-linear setting (mostly based on random feature models). Some of the latest works, such as \cite{mei2019generalization}, address the problem in the asymptotic regime. Our results provide an explicit control of the distance between some empirical estimators and the ground truth in terms of the subGaussian norms of the error and the covariance matrix of the covariate vectors. Our analysis employs efficient techniques from \cite{neuberger2007continuous}, combined with various results from Random Matrix Theory and high dimensional probability, as summarized in \cite{vershynin2010introduction} and extenstively discussed in \cite{vershynin2018high}. 
Theorem \ref{mainth_under_aniso} and Theorem \ref{mainth_over_aniso} provide a new finite sample analysis of the problem of estimating ridge functions in both underparametrised and overparametrised regimes, i.e. where the number of parameters is smaller (resp. larger) than the sample size. 

\begin{itemize}
    \item Our analysis of the underparametrised setting shows that we can obtain an error of order less than or equal to $\sqrt{p/n}$ for all $p$ up to an arbitrary fraction of $n$. 
    \item In the overparametrised setting, we get that the error bound is approaching the noise level plus $t \ K_Z \Vert A^\top \theta^*\Vert_2$ as $p$ grows to $+\infty$. This term can be small if the rank of $A$ is small, and even more so if $\theta^{*}$ is concentrated in the null space of $A$ or approximatively so, which is reminiscent of the results in \cite{tsigler2020benign} and the recent work \cite{lecue2022geometrical}.
\end{itemize}

Similar but simpler bounds also hold in the linear model setting, as presented in Corollary \ref{maincor_under_aniso} and Corollary \ref{maincor_over_aniso}. 

The following section addresses the problem of studying the prediction error in probability of the least $\ell_2$-norm empirical risk minimiser in the overparamatrised setting. 

\section{Generalisation of least norm empirical risk minimisers}

In this section, we leverage the results of the former section in order to study the prediction error of the least $\ell_2$-norm empirical risk minimiser. Our main result is Theorem \ref{thm:gene2_aniso} from the following section. 

\subsection{Main result}

Using the previous results, we obtain the following theorem about benign overfitting in the overparametrised case. 
\begin{theo}
\label{thm:gene2_aniso}
Assume that $\hat \theta$ is interpolating, i.e. $\hat R_n(\hat \theta)=0$. Let $\hat \theta^\circ$ denote the  minimum norm interpolating solution to the ERM problem, i.e. 
\begin{align}
    \text{argmin}_{\theta} \ \Vert \theta \Vert_2 \quad \text{ subject to } \hat R_n(\theta)=0.
    \label{thetasharpA}
\end{align}
Under the same assumptions as in Theorem \ref{mainth_over_aniso}, for any constant $\tau>0$, and for 
\begin{align}
    r & = \frac{C_{(\ell'',f',\varepsilon)} \sqrt{n}}{\delta\left( (1-\alpha) \sqrt{p} - C_{K_{Z}} \sqrt{n}\right)},
\end{align}
with $C_{(\ell'',f',\varepsilon)} = 6 \sqrt{C} C_{\ell''}C_{f'}K_{\varepsilon}$, where $C$ is a positive absolute constant, we have 
\begin{align}
\vert f(X_{n+1}^\top \hat \theta^\circ)  - f(X_{n+1}^\top\theta^*)\vert
& \le C_{f'} \left( \left\vert \varepsilon_{n+1}\right\vert + t \ K_Z \Vert A^\top \theta^*\Vert_2 + \sqrt{\tau (n/p) \log(n)}
    + t \ \frac{C_{(\ell'',f',\varepsilon)} \sqrt{n}}{ \delta ((1-\alpha) \sqrt{p} - C_{K_{Z}} \sqrt{n})} \right)
\label{titutututu_aniso}
\end{align}
with probability at least 
\begin{align*}
& 1 - 2 \exp{(-c_{K_{Z}} \alpha^{2} p)} - \exp{\left( -\frac{n}{2} \right)} - \exp ( -c_{K_{Z}} t^2) \\
&\hspace{1cm} - \exp\left(- \frac{t^2}{2}\right) - 4 \exp{\left( -c_{K_Z} \alpha p \right)} - \exp\left( - c C'^2 n \right) - 2 \exp \left(-\frac{\tau \log(n) \left(1-\alpha  - C_{K_{Z}} \sqrt{n/p}\right)^2}{8 K_Z^2 C'^2 K_{\varepsilon}^2 \kappa_n(A)^2 } \right)
\end{align*}
where $c$ and $C'$ are absolute positive constants. 
\end{theo}
\begin{proof}
See Section \ref{pf:gene2_aniso}.
\end{proof}

\begin{rem}
For isotropic data, take $A=I_{p}$. Notice that, in accordance with the concept of double descent as discussed in \cite{belkin2019reconciling}, when $n$ grows (while staying smaller than $p$ in the overparametrised regime), the ratio $n/p$ approaches $1$ from the right and the upper bound worsens.
\end{rem}

\subsection{Comparison with previous results}

Recently, the finite sample analysis has been addressed in the very interesting works \cite{bartlett2020benign}, \cite{chinot2020benign} and \cite{lecue2022geometrical} for the linear model only. These works propose very precise upper and lower bounds on the prediction risk for general covariate covariance matrices under the subGaussian assumption or the more restrictive Gaussian assumption. In \cite{bartlett2020benign}, excess risk is considered and the norm of $\Vert \theta^*\Vert$ appears in the proposed upper bound instead of $\Vert A^\top \theta^*\Vert_2$, which is potentially much smaller if $\theta^*$ has a large component in the kernel of $A$ or in spaces where the singular values of $A$ are small. The same issue arises in \cite{chinot2020benign}, where the noise is not assumed Gaussian nor subGaussian and can be dependent on the design matrix $X$. In \cite{lecue2022geometrical}, the authors propose a very strong bound that incorporates information that is more general than $\Vert A\theta^*\Vert_2$. However the setting studied in \cite{lecue2022geometrical} is restricted to linear models and Gaussian observation noise. Studies on  nonlinear models are scarce and a notable exception is \cite{frei2022benign}, where the setting is very different from ours and the number of parameters is constrained in a more rigid fashion than in the present work. 

In the present work, we show that similar, very precise results can be obtained for non-linear models of the ridge function class, using elementary perturbation results and some (now) standard random matrix theory. Moreover, our results concern the probability of exceeding a certain error level in predicting a new data point, whereas most results in the literature are restricted to the expected risk.

\section{Conclusion and perspectives}

This work presents a precise quantitative, finite sample analysis of the benign overfitting phenomenon in the estimation of linear and certain non-linear models. We make use of a zero-finding result of Neuberger \cite{neuberger2007continuous} which can be applied to a large number of settings in machine learning.

Extending our work to the case of Deep Neural Networks is an exciting avenue for future research. Another possible direction is to include penalisation, which can be addressed using the same techniques via Karush-Kuhn-Tucker conditions. Applying this approach to Ridge Regression and $\ell_1$-penalised estimation would bring new proof techniques into the field to investigate potentially deeper results. Weakening the assumptions on our data, which are here of subGaussian type, could also lead to interesting new results; this could be achieved by utilising, e.g. the work of \cite{mendelson2017extending}.

\bibliographystyle{plain}

\appendix

\section{Proofs of Theorem \ref{mainth_under_aniso} and Theorem \ref{mainth_over_aniso}}

\subsection{Neuberger's theorem}

The following theorem of Neuberger~\cite{neuberger2007continuous} will be instrumental in our study of the ERM. In our context, this theorem can be restated as follows.
\begin{theo}[Neuberger's theorem for ERM]
Suppose that $r > 0$, that $\theta^\ast \in \mathbb R^p$ and that the Jacobian $D \hat{R}_{n}(\cdot)$ is a continuous map on $B_{r}(\theta^\ast)$, with the property that for each $\theta$ in $b_{r}(\theta^\ast)$ there exists a vector $d$ in $B_{r}(0)$ such that, 
\begin{align}
\label{neubeq}
    \lim_{t \downarrow 0} \ \frac{ D \hat{R}_{n}(\theta + t d)  - D \hat{R}_{n}(\theta)}{t} = - D \hat{R}_{n}(\theta^*).
\end{align}
Then there exists $u$ in $B_{r}(\theta^\ast)$ such that $D \hat{R}_{n}(u) = 0$.
\label{thm:neub}
\end{theo}

\subsubsection{Computing the derivatives}
Since the loss is twice differentiable, the empirical risk $\hat{R}_{n}$ is itself twice differentiable. The Gradient of the empirical risk is given by 
\begin{align}
    \nabla \hat R_n (\theta) & = -\frac1{n} \sum_{i=1}^n \ 
    \ell'(Y_i-f(X_i^\top \theta)) \ f'(X_i^\top \theta)X_i.
    \label{gradRn}
\end{align}
Hence, for $\theta=\theta^*$, and recalling that $Y_i-f(X_i^\top \theta^*)=\varepsilon_i$, we get 
\begin{align}
    & = - \frac{1}{n} X^{\top} D_\nu \ \ell'(\varepsilon), \label{jac_vectorold}
\end{align}
where $\ell'(\varepsilon)$ is to be understood componentwise, and $D_\nu$ is a diagonal matrix with coefficients $\nu_i = f'(X_i^\top \theta^\ast)$ for all $i$ in $1, \ldots, n$. 

The Hessian is given by 
\begin{align}
    \nabla^2 \hat R_n (\theta) &= \frac1{n} \sum_{i=1}^n \ 
    \Big(\ell''(Y_i - f(X_i^\top \theta)) \ f'(X_i^t\theta)^{2} - \ell'(Y_i - f(X_i^\top \theta)) \ f''(X_i^\top \theta)\Big) X_i X_i^t \notag \\
    &= \frac{1}{n} \ X^\top D_\mu X, \label{hes_vectorold}
\end{align}
where $D_{\mu}$ is a diagonal matrix given by, for all $i$ in $1, \ldots, n$
\begin{align*}
    \mu_i & = \ell''(Y_i - f(X_i^\top \theta)) \ f'(X_i^t\theta)^{2}  - \ell'(Y_i - f(X_i^\top \theta)) \ f''(X_i^\top \theta).
\end{align*}
Notice that when the $\mu_1,\ldots,\mu_n$ are all non-negative for all $\theta \in \mathbb R^p$, the empirical risk is convex. The condition we have to satisfy in order to use  Neuberger's theorem, i.e. the version of \eqref{neubeq} associated with our setting, is the following
\begin{align*}
    \nabla^2 \hat R_n(\theta) d & = -\nabla \hat R_n(\theta^\ast).
\end{align*}

\label{mainproof_under}

\subsubsection{A change of variable}

Since 
\begin{align}
    \hat R_n(\theta) 
& = \frac{1}{n} \sum_{i=1}^{n} \ell(Y_{i}- f(X_{i}^\top \theta)).
\end{align}
and $X_{i}^\top = Z_{i}^\top A^\top$, the risk can be written as a function of $A\theta$ as follows: 
\begin{align}
    \hat R^{(A)}_n(\theta^{(A)}) 
& = \frac{1}{n} \sum_{i=1}^{n} \ell(Y_{i}- f(Z_{i}^\top \theta^{(A)}))
\end{align}
with $\theta^{(A)} = A^\top \theta$. Clearly, minimizing $\hat R_n$ in $\theta$ is equivalent to minimizing $\hat R^{(A)}_n$ in $\theta^{(A)}$ up to equivalence modulo the kernel of $A^\top$. 

\begin{align}
    \nabla \hat R^{(A)}_n (\theta^{(A)}) & = -\frac1{n} \sum_{i=1}^n \ 
    \ell'(Y_i-f(Z_i^\top\theta^{(A)})) \ f'(Z_i^\top  \theta^{(A)}) Z_i.\\
    & = - \frac{1}{n} Z^{\top} D_\nu \ \ell'(\varepsilon), \label{jac_vector}
\end{align}

\begin{align}
    \nabla^2 \hat R^{(A)}_n (\theta^{(A)}) &= \frac1{n} \sum_{i=1}^n \ 
    \Big(\ell''(Y_i - f(Z_i^\top \theta^{(A)})) \ f'(Z_i^\top \theta^{(A)})^{2} - \ell'(Y_i - f(Z_i^\top\theta^{(A)})) \ f''(Z_i^\top\theta^{(A)})\Big) Z_i Z_i^\top \notag \\
    &= \frac{1}{n} \ Z^\top D_\mu Z, \label{hes_vector}
\end{align}
where $D_{\mu}$ is a diagonal matrix given by, for all $i$ in $1, \ldots, n$
\begin{align*}
    \mu_i & = \ell''(Y_i - f(Z_i^\top \theta^{(A)})) \ f'(Z_i^\top \theta^{(A)})^{2}  - \ell'(Y_i - f(Z_i^\top \theta^{(A)})) \ f''(Z_i^\top \theta^{(A)}).
\end{align*}
The Neuberger equation in $\theta^{(A)}$ is now given by 
\begin{align}
\label{neubeqA}
    \lim_{t \downarrow 0} \ \frac{ D \hat{R}^{(A)}_{n}(\theta^{(A)} + t d)  - D \hat{R}^{(A)}_{n}(\theta^{(A)})}{t} = - D \hat{R}_{n}(\theta^{(A)^*}),
\end{align}
where $\theta^{(A)^*} = A^\top \theta^*$.

\subsubsection{A technical lemma}

\begin{lem}
\label{techlemsubgaussellp}
For all $i=1,\ldots,n$, the variable $\ell'(\varepsilon_i)$ is subGaussian with variance proxy $\Vert \ell'(\varepsilon_{i}) \Vert_{\psi_{2}} = C_{\ell''} \ \Vert \varepsilon_i \Vert_{\psi_2}.$  
\end{lem}

\begin{proof}
Let us compute 
\begin{align*}
    \Vert \ell'(\varepsilon_i) \Vert_{\psi_2} & = \sup _{\gamma \geq 1} \gamma^{-1 / 2}\Big(\mathbb{E}|\ell'(\varepsilon_i)|^{\gamma}\Big)^{1 / \gamma}.
\end{align*}
Lipschitzianity of $\ell'$ implies that 
\begin{align*}
    \vert \ell'(\varepsilon_i)-\ell'(0)\vert & \le C_{\ell''} \vert \varepsilon_i-0 \vert
\end{align*}
and since $\ell'(0)=0$, we get 
\begin{align*}
    \vert \ell'(\varepsilon_i)\vert & \le C_{\ell''} \vert \varepsilon_i \vert,
\end{align*}
which implies that 
\begin{align*}
    \Vert \ell'(\varepsilon_i) \Vert_{\psi_2} & = C_{\ell''} \ \sup _{\gamma \geq 1} \gamma^{-1 / 2}\Big(\mathbb{E}|\varepsilon_i|^{\gamma}\Big)^{1 / \gamma}.
\end{align*}
Thus, 
\begin{align*}
    \Vert \ell'(\varepsilon_i) \Vert_{\psi_2} & = C_{\ell''} \ \Vert \varepsilon_i \Vert_{\psi_2} < +\infty
\end{align*}
as announced. The proof is complete.
\end{proof}

\subsection{Proof of Theorem \ref{mainth_under_aniso}: The underparametrised case}

\label{mainproof_under_aniso}

\subsubsection{Key lemma}
Using Corollary \ref{vershynin_cor}, we have $s_{n}(Z)>0$ as long as 
\begin{align*}
    C_{K_Z}^2 p < (1-\alpha)^2 n
\end{align*}
for some positive constant $C_{K_Z}$ depending on the subGaussian norm $K_Z$ of $Z_1,\ldots,Z_n$.

\begin{lem}
With probability larger than or equal to 
$1 - \exp{\Big( -\frac{p}{2} \Big)},$
we have
\begin{align*}
    \Vert d\Vert_2 & \leq \frac{C_{(\ell'',f',\varepsilon)} \sqrt{p}}{ \delta s_{n}(Z) },
\end{align*}
where $C_{(\ell'',f',\varepsilon)} = 6 \sqrt{C} C_{\ell''} C_{f'}  K_{\varepsilon}$.

\end{lem}

\begin{proof}
We compute the solution of Neuberger's equation \eqref{neubeqA}, using the Jacobian vector formula in~\eqref{jac_vector} and the Hessian matrix formula in~\eqref{hes_vector}: 

\begin{align*}
    \nabla^2\hat R^{(A)}_n(\theta^{(A)}) d & = - \nabla \hat R^{(A)}_n(A^\top\theta^*),
\end{align*}
which gives 
\begin{align*}
    Z^\top D_\mu Z d & = Z^\top D_\nu \ell'(\varepsilon). 
\end{align*}
The singular value decomposition of $Z$ gives $Z = U \Sigma V^\top$, where $U \in \mathbb R^{n\times p}$ is a matrix whose columns form an orthonormal family, $V\in \mathbb R^{p\times p}$ is an orthogonal matrix and $\Sigma \in \mathbb R^{p\times p}$ is a diagonal and invertible matrix. 
Thus, we obtain the equivalent equation
\begin{align*}
    V \Sigma U^\top D_\mu U \Sigma V^\top d & =  V \Sigma U^\top D_\nu \ell'(\varepsilon).
\end{align*}
Note that any solution of the system
\begin{align*}
    U^\top D_\mu U \Sigma V^\top d & =   U^\top D_\nu \ell'(\varepsilon),
\end{align*}
will be admissible. Hence, 
\begin{align*}
     V^\top  d & = \Sigma^{-1}(U^\top D_\mu U)^{-1} U^\top D_\nu \ell'(\varepsilon).
\end{align*}
Then, as $\Vert d \Vert_2 = \Vert V^\top d \Vert_2$,
\begin{align}
    \big\Vert d \big\Vert_{2} & \le \big\Vert \Sigma^{-1}(U^\top D_\mu U)^{-1} U^\top D_\nu \ell'(\varepsilon) \big\Vert_{2} \nonumber \\
    & \leq \big\Vert \Sigma^{-1} \Big\Vert_{2} \Big\Vert (U^\top D_\mu U)^{-1} U^\top D_\nu \ell'(\varepsilon) \big\Vert_{2} \nonumber \\
    & \leq \frac1{s_{n}(Z)} \Vert U^\top \ D_{\mu}^{-1} D_\nu \ell'(\varepsilon) \Vert_2.
    \label{proof::norm_d_overA}
\end{align}

Using maximal inequalities, we can now bound the deviation probability of 
\begin{align*}
    \Vert U^\top \ D_{\mu}^{-1} D_\nu \ell'(\varepsilon) \Vert_2 & = \max_{\Vert u \Vert_2=1} \ u^\top U^\top D_{\mu}^{-1} D_\nu \ell'(\varepsilon),
\end{align*}
with $u\in \mathbb{R}^{p}$. For this purpose, we first prove that $U^\top D_{\mu}^{-1} D_\nu \ell'(\varepsilon)$ is a subGaussian vector and provide an explicit upper bound on its norm. Since $U$ is a $n \times p$ matrix whose columns form an orthonormal family, $u^\top U^\top$ is a unit-norm $\ell_2$ vector of size $n$ which is denoted by $w$. Then, 
\begin{align*}
    u^\top U^\top D_{\mu}^{-1} D_\nu \ell'(\varepsilon) & = \sum_{i=1}^n \ w_i \frac{\nu_i}{\mu_i} \ \ell'(\varepsilon_i),
\end{align*}
and since $\ell'(\varepsilon_i)$ is centered  and subGaussian for all $i=1,\ldots,n$, Vershynin's \cite[Proposition 2.6.1]{vershynin2018high} gives 
\begin{align*}
    \Vert u^\top U^\top D_{\mu}^{-1} D_\nu \ell'(\varepsilon)\Vert_{\psi_2}^2 & \le C \sum_{i=1}^n \ \left \| w_i \frac{\nu_i}{\mu_i} \ \ell'(\varepsilon_i) \right \|_{\psi_{2}}^2,
\end{align*}
for some absolute constant $C$.
Since 
\begin{align*}
    \left \| w_i \frac{\nu_i}{\mu_i} \ell'(\varepsilon_i) \right \|_{\psi_2}^2 & \le 
    \vert w_i \vert^2\ \max_{i'=1}^n \ \left \| \frac{\nu_{i'}}{\mu_{i'}} \ \ell'(\varepsilon_{i'}) \right \|_{\psi_2}^2,
\end{align*}
we have 
\begin{align*}
    \Vert u^\top U^\top D_{\mu}^{-1} D_\nu \ell'(\varepsilon)\Vert_{\psi_2}^2 & \le C \Vert w\Vert_2^2 \ \max_{i'=1}^n \ \left \| \frac{\nu_{i'}}{\mu_{i'}} \ \ell'(\varepsilon_{i'}) \right \|_{\psi_2}^2.
\end{align*}
As $\Vert w\Vert_2=1$, we get that for all $u \in \mathbb R^p$ with $\Vert u\Vert_2=1$,
\begin{align*}
    \Vert u^\top U^\top D_{\mu}^{-1} D_\nu \ell'(\varepsilon)\Vert_{\psi_2}^2 
    & \le C \max_{i'=1}^n \ \left \| \frac{\nu_{i'}}{\mu_{i'}} \ \ell'(\varepsilon_{i'}) \right \|_{\psi_2}^2.
\end{align*}
We deduce that $U^\top D_{\mu}^{-1} D_\nu \ell'(\varepsilon)$ is a subGaussian random vector with variance proxy
\begin{align*}
C \max_{i'=1}^n \ \left \| \frac{\nu_{i'}}{\mu_{i'}} \ \ell'(\varepsilon_{i'}) \right \|_{\psi_2}^2.
\end{align*}
Using the maximal inequality from~\cite[Theorem 1.19]{rigollet2015high} and the subGaussian properties described in~\cite[Proposition 2.5.2]{vershynin2018high}, for any $\eta>0$, we get with probability $1 - \exp{\Big( -\frac{\eta^{2}}{2} \Big)}$ 
\begin{align}
    \Vert U^\top \ D_{\mu}^{-1} D_\nu \ell'(\varepsilon) \Vert_2 &\le 2 \sqrt{C} \max_{i'=1}^n \ \left \| \frac{\nu_{i'}}{\mu_{i'}} \ \ell'(\varepsilon_{i'}) \right \|_{\psi_2} (2 \sqrt{p} + \eta).
    \label{proof::main_normA}
\end{align}
Following Lemma~\ref{techlemsubgaussellp}, we deduce that $\frac{\nu_{i}}{\mu_{i}} \ell'(\varepsilon_{i'})$ is a subGaussian random variable with variance proxy 
\begin{align*}
    \left \| \frac{\nu_{i'}}{\mu_{i'}} \ell'(\varepsilon_{i'}) \right \|_{\psi_{2}} & =  \frac{\vert \nu_{i'} \vert}{\vert \mu_{i'}\vert} C_{\ell''} \Vert \varepsilon_{i'} \Vert_{\psi_{2}} \\
    & \leq C_{\ell''} \frac{\max_{i'}^{n} \vert \nu_{i'} \vert}{\min_{i'}^{n} \vert \mu_{i'}\vert} K_{\varepsilon}.
\end{align*}
Then, restricting  a priori $\theta^{(A)}$ to lie in the ball $\mathcal B_2(\theta^{(A)^*},r)$, which is coherent with the assumptions of Newberger's Theorem \ref{thm:neub}, and using the assumption \ref{assump_lfcond_aniso}, together with the boundedness of $f'$, we get
\begin{align*}
    \frac{\max_{i'=1}^{n} \vert \nu_{i'} \vert}{\min_{i'=1}^{n} \vert \mu_{i'} \vert} \leq \frac{C_{f'}}{\delta}.
\end{align*}
Subsequently, the quantity~\eqref{proof::main_normA} becomes
\begin{align*}
    \Vert U^\top \ D_{\mu}^{-1} D_\nu \ell'(\varepsilon) \Vert_2 & \le 2 \sqrt{C} C_{\ell''} \frac{\max_{i'=1}^{n} \vert \nu_{i'} \vert}{\min_{i'=1}^{n} \vert \mu_{i'} \vert} K_{\varepsilon} (2 \sqrt{p} + \eta) \nonumber \\
    & \leq \frac{2 \sqrt{C} C_{\ell''} C_{f'}  K_{\varepsilon} (2 \sqrt{p} + \eta)}{\delta},
\end{align*}
with probability $1 - \exp{\left( -\frac{\eta^{2}}{2} \right)}$. 
Taking $\eta = \sqrt{p}$, equation \eqref{proof::norm_d_overA} yields
\begin{align*}
    & \Vert d\Vert_2 = \Big\Vert \nabla^2\hat R_n(\theta)^{-1}\ \nabla \hat R_n(\theta^*) \Big\Vert_{2}  \leq \frac{6 \sqrt{C} C_{\ell''}C_{f'}  K_{\varepsilon} \sqrt{p}}{\delta \ s_{n}(Z)},
\end{align*}
with probability $1 - \exp{\left( -\frac{p}{2} \right)}$. 
\end{proof}

\subsubsection{End of the proof of Theorem \ref{mainth_under_aniso}}

In order to complete the proof of Theorem \ref{mainth_under_aniso}, note that, using Corollary~\ref{vershynin_cor}, we have with probabilty $1 - 2 \exp{(-c_{K_{Z}} \alpha^{2} n)}$
\begin{align*}
    s_{n}(Z) \geq (1-\alpha)\sqrt{n} -  C_{K_{Z}} \sqrt{p},
\end{align*}
where $C_{K_{Z}}, c_{K_{Z}} > 0$ depend only on the subGaussian norm $K_{Z}$. 
Therefore,  with probability larger than or equal to 
\begin{align*}
    1 & - 2 \exp{(-c_{K_{Z}} \alpha^{2} n)} - \exp{\left( -\frac{p}{2} \right)},
\end{align*}
we have
\begin{align*}
    \Vert d\Vert_2 & \leq \frac{6 \sqrt{C} C_{\ell''} C_{f'} K_{\varepsilon} \sqrt{p}}{ \delta \left((1-\alpha) \sqrt{n} - C_{K_{Z}} \sqrt{p} \right)}.
\end{align*}
Hence, the proof of Theorem \ref{mainth_under_aniso} is completed.

\subsection{Proof of Theorem \ref{mainth_over_aniso}: The overparametrised case}
\label{mainproof_over_aniso}

\subsubsection{Key lemma}

Using Theorem \ref{vershcolumn}, we have $s_{n}(Z^\top)>0$ as long as 
\begin{align*}
    C_{K_Z}^2 \ n < (1-\alpha)^2 p.
\end{align*}
for some positive constant $C_{K_Z}$ depending on the subGaussian norm $K_Z$ of $Z_1,\ldots,Z_n$.

\begin{lem}
With probability larger than or equal to $1 - \exp{\left( -\frac{n}{2} \right)}$, we have 
\begin{align*}
    \Vert d\Vert_2 & \leq \frac{C_{(f'',\ell',\varepsilon)} \sqrt{n}}{\delta \ s_{n}(Z^{\top}) }.
\end{align*} 
where $C_{(f'',\ell',\varepsilon)} = 6 \sqrt{C} C_{\ell''} C_{f'}  K_{\varepsilon}$.
\end{lem}

\begin{proof}


As in the underparametrised case, we have to solve \eqref{neubeq} i.e 
\begin{align*}
    Z^\top D_\mu Z d = Z^\top D_\nu \ell'(\varepsilon),
\end{align*}
which can be solved by finding the least norm solution of the interpolation problem
\begin{align}
    D_\mu Z d & = D_\nu \ell'(\varepsilon),
\end{align}
i.e. 
\begin{align*}
    d & = Z^\dagger D_{\mu}^{-1} D_\nu \ell'(\varepsilon)
\end{align*}
where $Z^\dagger$ is the Moore-Penrose pseudo-inverse of $Z$.

Given the compact SVD of $Z=U \Sigma V^\top$, where $U\in O(n)$ and $V \in \mathbb R^{p\times n}$ with orthonormal columns, we get 
\begin{align*}
    D_\mu U \Sigma V^\top d & =   D_\nu \ell'(\varepsilon),
\end{align*}
We then have 
\begin{align*}
    \Vert V^\top d \Vert_2 & = \Vert  \Sigma^{-1} U^\top D_{\mu}^{-1} D_\nu \ell'(\varepsilon) \Vert_2.
\end{align*}
As $\Vert x \Vert_{2} = \Vert V x \Vert_{2}$ for all $x \in \mathbb R^n$,
\begin{align*}    
    \Vert d \Vert_2 &\le \frac{1}{s_{n}(Z^\top)} \Vert U^\top D_{\mu}^{-1} D_\nu \ell'(\varepsilon) \Vert_2.
\end{align*}
Then, using the bound
\begin{align}
    & \Vert U^\top \ D_{\mu}^{-1} D_\nu \ell'(\varepsilon) \Vert_2 \nonumber  \le \frac{2 \sqrt{C} C_{\ell''} C_{f'} K_{\varepsilon} (2 \sqrt{n} + \eta)}{\delta}
    \label{proof::final_eqA}
\end{align}
with probability $1 - \exp{\left( -\frac{\eta^{2}}{2} \right)}$, which can be obtained in the same way as \eqref{proof::main_normA} for the underparametrised case, and 
taking $\eta = \sqrt{n}$, we get
\begin{align}
    \Vert d\Vert_2 & \leq \frac{6 \sqrt{C} C_{\ell''} C_{f'}  K_{\varepsilon} \sqrt{n}}{\delta \ s_{n}(Z^\top)},
\end{align} 
with probability $1 - \exp{\left( -\frac{n}{2} \right)}$.
\end{proof}

\subsubsection{End of the proof of Theorem \ref{mainth_over_aniso}}

In order to complete the proof of Theorem \ref{mainth_over_aniso}, note that, using Theorem~\ref{vershcolumn}, we have with probability $1 - 2 \exp{(-c_{K_{Z}} \alpha^{2} p)}$
\begin{align}
    s_{n}(Z^\top) \geq (1-\alpha) \sqrt{p} - C_{K_{Z}} \sqrt{n}.
\end{align}
Therefore,  with probability larger than or equal to 
\begin{align*}
    1 & - 2 \exp{(-c_{K_{Z}} \alpha^{2} p)} - \exp{\left( -\frac{n}{2} \right)},
\end{align*}
we have
\begin{align*}
    \Vert d\Vert_2 & \leq \frac{6 \sqrt{C} C_{\ell''} C_{f'} K_{\varepsilon} \sqrt{n}}{ \delta ((1-\alpha) \sqrt{p} - C_{K_{Z}}\sqrt{n})}.
\end{align*} 
Hence the proof of Theorem \ref{mainth_over_aniso} is completed.

\section{Proof of Theorem \ref{thm:gene2_aniso}}
\label{pf:gene2_aniso}

\subsection{First part of the proof}
Since $\hat \theta$ and $\hat \theta^\circ$ are two interpolating minimisers of the empirical risk, i.e. satisfy $Y_i-f(X_i^\top \hat \theta)=0$ and $Y_i-f(X_i^\top \hat \theta^\circ)=0$
for all $i$ in $1, \ldots, n$, then using that $f$ is strictly monotonic, we get
 \begin{align}
     X\hat \theta^\circ & = X\hat \theta. 
 \end{align}
\subsection{Second part of the proof}
Recall that $\hat \theta^\circ$ denote the  minimum norm solution to the ERM, i.e. 
\begin{align*}
    \text{argmin}_{\theta} \ \Vert \theta \Vert_2 \quad \text{ subject to }           X\theta = X\hat \theta.
\end{align*}
Let $\hat\theta^\sharp$ solves
\begin{align*}
    \text{argmin}_{\theta} \ \Vert \theta \Vert_2 \quad \text{ subject to } \begin{bmatrix}
        X \\
        X_{n+1}^\top
    \end{bmatrix}\theta = \begin{bmatrix}
        X \\
        X_{n+1}^\top
    \end{bmatrix}\hat \theta,
\end{align*}
where $\hat \theta$ is the solution exhibited in Theorem \ref{mainth_over_aniso}.
Then, 
\begin{align}
    \vert X_{n+1}^\top (\hat \theta^\circ -\theta^*) \vert & \le \vert X_{n+1}^\top (\hat \theta^\circ -\hat \theta^\sharp) \vert +  \vert 
    \underbrace{X_{n+1}^\top (\hat \theta^\sharp -\hat \theta)}_{=0 \text{ by definition}} \vert+ 
    \vert X_{n+1}^\top (\hat \theta -\theta^*) \vert, \nonumber \\
    & \le \vert X_{n+1}^\top (\hat \theta^\circ -\hat \theta^\sharp) \vert+\vert X_{n+1}^\top (\hat \theta -\theta^*) \vert
    \label{mainineg}
\end{align}
and since $X_{n+1}$ and  $\hat \theta -\theta^*$ are independent, we have 
\begin{align*}
    \mathbb P( \vert X_{n+1}^\top (\hat \theta -\theta^*) \vert\ge t \Vert A^\top (\hat \theta -\theta^*)\Vert_2)  & = \mathbb P( \vert Z_{n+1}^\top A^\top (\hat \theta -\theta^*) \vert\ge t \Vert A^\top (\hat \theta -\theta^*)\Vert_2) \\
    & \le \exp (-c_{K_Z} t^2).
\end{align*}
Now, by Theorem \ref{mainth_over_aniso}, about the bound on $\Vert A^\top (\hat \theta -\theta^*)\Vert_2$, we get 
\begin{align}
    \vert X_{n+1}^\top (\hat \theta -\theta^*) \vert & \le t \ \frac{6\sqrt{C} C_{\ell''} C_{f'} K_{\varepsilon} \sqrt{n}}{ \delta ((1-\alpha) \sqrt{p} - C_{K_{Z}} \sqrt{n}) },
    \label{titutu}
\end{align}
with probability at least 
\begin{align*}
    1 & -  2 \exp{(-c_{K_{Z}} \alpha^{2} p)}  - \exp{\left( -\frac{n}{2} \right)} \textcolor{black}{- \exp ( -c_{K_Z} t^2)}.
\end{align*}

\subsection{Third part of the proof: Block pseudo-inverse computation}
\label{pseud}
Since $\hat \theta^{\sharp}$ is the minimum $\ell_2$-norm solution to the under-determined system
\begin{align*}
\left[\begin{array}{l}
X \\
X_{n+1}^{\top}
\end{array}\right] \theta=\left[\begin{array}{l}
y_{1:n} \\
y_{n+1}
\end{array}\right],
\end{align*}
then we have 
\begin{align*}
\hat \theta^\sharp &  =\left(X \left(I -\frac1{p}X_{n+1}X_{n+1}^\top\right)\right)^{\dagger} y_{1:n}+\left(X_{n+1}^{\top} \left( I-X^\top \left(XX^\top\right)^{\dagger}X\right)\right)^{\dagger} y_{n+1}.
\end{align*}
This gives
\begin{align}
\label{hatthetasharp}
\hat \theta^\sharp &  = \left(I -\frac1{p}X_{n+1}X_{n+1}^\top\right) X^\dagger y_{1:n}+ \left( I-X^\top \left(XX^\top\right)^{\dagger}X\right) X_{n+1}^{\top^\dagger} y_{n+1},
\end{align}
since 
\begin{align*}
    \left(I -\frac1{p}X_{n+1}X_{n+1}^\top\right)^{\dagger}=\left(I -\frac1{p}X_{n+1}X_{n+1}^\top\right)
\end{align*} 
and 
\begin{align*}
    \left(X_{n+1}^{\top} \left( I-X^\top \left(XX^\top\right)^{-1}X\right)\right)^{\dagger}= \left(X_{n+1}^{\top} \left( I-X^\top \left(XX^\top\right)^{\dagger}X\right)\right).
\end{align*}

\subsection{Fourth part of the proof}
Since $\hat \theta^\circ = X^\dagger y_{1:n}$,  we get from the expression of $\hat \theta^\sharp$ in \eqref{hatthetasharp} that 
\begin{align*}
\vert X_{n+1}^\top (\hat \theta^\sharp - \hat \theta^\circ)\vert & =  \left\vert -\frac1{p}X_{n+1}^\top X_{n+1}X_{n+1}^\top X^\dagger y_{1:n}+ X_{n+1}^\top \left( I-X^\top \left(XX^\top\right)^{\dagger}X\right) X_{n+1}^{\dagger} y_{n+1}\right\vert\\
& \le  \left\vert X_{n+1}^\top X^\top (XX^\top)^{\dagger} X \theta^*  +  X_{n+1}^\top \left( I-X^\top \left(XX^\top\right)^{\dagger}X\right)X_{n+1}^{\dagger} X_{n+1}^\top \theta^*\right\vert \\
& \hspace{2cm} + \left\vert X_{n+1}^\top X^\top (XX^\top)^{\dagger} \varepsilon \right\vert+ \left\vert   \varepsilon_{n+1}\right\vert \\
& \le  \left\vert X_{n+1}^\top X^\top (XX^\top)^{\dagger} X \theta^*  +  X_{n+1}^\top \left( I-X^\top \left(XX^\top\right)^{\dagger}X\right) \theta^*\right\vert \\
& \hspace{2cm} + \left\vert X_{n+1}^\top X^\top (XX^\top)^{\dagger} \varepsilon \right\vert+ \left\vert    \varepsilon_{n+1}\right\vert, 
\end{align*}
which gives 
\begin{align}
\vert X_{n+1}^\top (\hat \theta^\sharp - \hat \theta^\circ)\vert
& \le  \left\vert  X_{n+1}^\top \theta^*\right\vert+ \left\vert X_{n+1}^\top X^\top (XX^\top)^{\dagger} \varepsilon \right\vert+ \left\vert    \varepsilon_{n+1}\right\vert.
\label{predineg_aniso}
\end{align}

\subsection{Fifth part of the proof}

In this section, we control of each term in the RHS of \eqref{predineg_aniso}. 

\vspace{.5cm}

\begin{center}
\textit{\bf Control of  $\left|X_{n+1}^{\top} \theta^{*}\right|$}
\end{center}
We have 
\begin{align}
    \mathbb P(\vert \langle \theta^*,X_{n+1}\rangle\vert \ge t \ K_X \Vert \theta^*\Vert_2) = \mathbb P(\vert \langle A^\top \theta^*, Z_{n+1}\rangle\vert \ge t \ K_Z \Vert A^\top \theta^*\Vert_2)   & \le \exp\left(- \frac{t^2}{2}\right).
        \label{tatitoto2_aniso}
\end{align}

\vspace{.5cm}

\begin{center}
\textit{\bf Control of $\left\vert X_{n+1}^\top X^\top (XX^\top)^{\dagger} \varepsilon \right\vert$}
\end{center}

Let us compute an upper bound on the norm of $X^\top (XX^\top)^{\dagger} \varepsilon$ which holds with large probability. We have 
\begin{align*}
    \Vert X^\top (XX^\top)^{\dagger} \varepsilon \Vert_2^2 & = 
    \varepsilon^\top (XX^\top)^{\dagger} \varepsilon = \varepsilon^\top (Z A^\top AZ^\top)^{\dagger} \varepsilon 
\end{align*}

Our main tool will be the following result from Vershynin's book \cite[Exercise 6.3.5]{vershynin2018high}

\begin{align}
\mathbb{P}\left(\|(XX^\top)^{\dagger^{1/2}} \varepsilon \|_{2} \geq C K_\varepsilon \|(XX^\top)^{\dagger^{1/2}} \|_{F}+t \mid X \right) \leq \exp \left(-\frac{c t^{2}}{K_{\varepsilon}^{2}\|(XX^\top)^{\dagger^{1/2}}\|^{2}}\right),
\label{vershyninquad}
\end{align}
where $C'$ and $c$ are absolute positive constants.

Let us introduce the Singular Value Decomposition of 
\begin{align*}
A = U_{(A)} \Sigma_{(A)} V_{(A)}^\top    
\end{align*}
with $\Sigma_{(A})$ is a $r_{(A)} \times r_{(A)}$ matrix, where $r_{(A)}$ is the rank of $A$. 
We then have
\begin{align*}
    X X^\top &= Z A^\top A Z^\top = Z (U_{(A)} \Sigma_{(A)} V_{(A)}^\top)^\top (U_{(A)} \Sigma_{(A)} V_{(A)}^\top) Z^\top = Z V_{(A)} \Sigma_{(A)} U_{(A)}^\top U_{(A)} \Sigma_{(A)} V_{(A)}^\top Z^\top \notag \\
    &= Z V_{(A)} \Sigma_{(A)}^2 V_{(A)}^\top Z^\top \quad \text{because $U$ has orthogonal columns}
\end{align*}

Then we have
\begin{align*}
    \Vert (XX^\top)^{\dagger^{1/2}}\Vert_F^2 &= \left\Vert \left( Z V_{(A)} \Sigma_{(A)}^2 V_{(A)}^\top Z^\top \right)^{\dagger^{1/2}} \right\Vert_{F}^2 = \sum_{i=1}^n \lambda_{i} \left( \left( Z V_{(A)} \Sigma_{(A)}^2 V_{(A)}^\top Z^\top \right)^{\dagger} \right),
\end{align*}
and since $\lambda_{i}((BB^\top)) = s_i(B)^{2}$ for some matrix $B$: 
\begin{align*}
\Vert (XX^\top)^{\dagger^\frac12}\Vert_F^2 &= \sum_{i=1}^n s_{i} \left( Z V_{(A)} \Sigma_{(A)} \right)^{-2} \\
&\leq n \ s_{n} (\Sigma_{(A)} V_{(A)}^T Z^T )^{-2} \\
& \leq n \ s_{n}(\Sigma_{(A)})^{-2} \ s_{n} (V_{(A)}^\top Z^\top)^{-2}
\end{align*}
where $s_n$ denotes the $n^{th}$ largest singular value. Since the columns of $V_{(A)}$ form an othornomal family, we get 
\begin{align*}
\Vert (XX^\top)^{\dagger^\frac12}\Vert_F^2
&\leq n \ s_{n} (Z^\top )^{-2} s_{n} (\Sigma_{(A)})^{-2} \leq n \ s_{n} (Z^\top)^{-2} s_{n}(A)^{-2}.
\end{align*}

Hence
\begin{align*}
    \Vert (XX^\top)^{\dagger^{1/2}}\Vert_F & \le \sqrt{n} \ s_{n} (Z^\top)^{-1} s_{n}(A)^{-1}. \\
\end{align*}

Since $(X X^\top)^{\dagger^{1/2}}$ is symmetric, we have
\begin{align*}
    \Vert (XX^\top)^{\dagger^{1/2}}\Vert &= \lambda_{1} ((XX^\top)^{\dagger^{1/2}}) = \lambda_{1}(((Z V_{(A)} \Sigma_{(A)}^2 V_{(A)}^\top Z^\top)^{\dagger})^{1/2}) \\
    &= s_{1}((Z V_{(A)} \Sigma_{(A)})^{\dagger}) \\
    &= \frac{1}{s_{n}(\Sigma_{(A)} V_{(A)} Z^T)} \\
    &\le \frac{1}{s_{n}(\Sigma_{(A)}) s_n(Z^T)} \\
    &= \frac{1}{s_{n}(Z^T) s_n(A)}.
\end{align*}

Define the event 
\begin{align*}
    \mathcal E_{Z} & = \left\{ s_{n}\left(Z^\top\right) \ge  \left((1-\alpha) \sqrt{p} - C_{K_{Z}} \sqrt{n}\right) \right\}.
\end{align*}

Then, following Theorem \ref{vershcolumn}
\begin{align*}
    \mathbb P(\mathcal E_Z^c) & \le 2 \exp \left(-c_{K_Z} \alpha p\right). 
\end{align*}
Let us notice that, on this event $\mathcal{E}_Z$, we have
\begin{align*}
    \Vert (XX^\top)^{\dagger^{1/2}}\Vert_F &\le \frac{\sqrt{n}}{s_{n}(A) \ \left((1-\alpha) \sqrt{p} - C_{K_{Z}} \sqrt{n}\right)} \\
    \Vert (XX^\top)^{\dagger^{1/2}}\Vert &\le \frac{1}{s_{n}(A) \ \left((1-\alpha) \sqrt{p} - C_{K_{Z}} \sqrt{n}\right)}. 
\end{align*}

Using \eqref{vershyninquad} after setting 
\begin{align*}
    t & = \sqrt{n}\ \frac{C' K_{\varepsilon}}{s_{n}(A) \left((1-\alpha)\sqrt{p}-C_{K_Z} \sqrt{n}\right)},
\end{align*} and conditioning on $Z$, we get 
\begin{align}
\mathbb{P}_\varepsilon\left(\|(XX^\top)^{\dagger^{1/2}} \varepsilon \|_{2} \geq \frac{2 C' K_\varepsilon \sqrt{n} }{s_{n}(A) \left((1-\alpha) \sqrt{p} - C_{K_{Z}} \sqrt{n}\right)} \Big| Z \right) \leq \exp{\left( - c C'^2 n \right)},
\label{tititutu}
\end{align}
so long as $Z\in  \mathcal E_{Z}$.

Let us define 
\begin{align*}
    \mathcal E_{\varepsilon,Z} & = 
    \left\{ \varepsilon,Z \mid \|(XX^\top)^{\dagger^{1/2}} \varepsilon \|_{2} \le \frac{2 C' K_\varepsilon \sqrt{n} }{s_{n}(A) \left((1-\alpha) \sqrt{p} - C_{K_{Z}} \sqrt{n}\right)} \text{ and } Z \in \mathcal E_Z\right\}, 
\end{align*}
then
\begin{align*}
    \mathcal E_{\varepsilon,Z}^c & = 
    \left\{ \varepsilon,Z \mid \|(XX^\top)^{\dagger^{1/2}} \varepsilon \|_{2} \geq \frac{2 C' K_\varepsilon \sqrt{n} }{s_{n}(A) \left((1-\alpha) \sqrt{p} - C_{K_{Z}} \sqrt{n}\right)} \cup Z \notin \mathcal E_Z\right\}. 
\end{align*}

Notice that
\begin{align*}
    \mathbb P_{\varepsilon,Z} \left(\mathcal E_{\varepsilon,Z} \right) & = 
    \mathbb E_{\varepsilon,Z} \left[ 1_{\|(XX^\top)^{\dagger^{1/2}} \varepsilon \|_{2} \le \frac{2 'C K_\varepsilon \sqrt{n} }{s_{n}(A) \left((1-\alpha) \sqrt{p} - C_{K_{Z}} \sqrt{n}\right)}} 1_{Z \in \mathcal E_Z} \right] \\
    & = \mathbb E_{Z} \left[\mathbb E_{\varepsilon} \left[ 1_{\|(XX^\top)^{\dagger^{1/2}} \varepsilon \|_{2} \le \frac{2 C' K_\varepsilon \sqrt{n} }{s_{n}(A) \left((1-\alpha) \sqrt{p} - C_{K_{Z}} \sqrt{n}\right)}} 1_{Z \in \mathcal E_Z} \Big| Z\right]\right] \\
    & = \mathbb E_{Z} \left[\mathbb E_{\varepsilon} \left[ 1_{\|(XX^\top)^{\dagger^{1/2}} \varepsilon \|_{2} \le \frac{2 C' K_\varepsilon \sqrt{n} }{s_{n}(A) \left((1-\alpha) \sqrt{p} - C_{K_{Z}} \sqrt{n}\right)}} \Big| Z\right]  1_{Z \in \mathcal E_Z} \right] \\
    & = \mathbb E_{Z} \left[\mathbb P_{\varepsilon} \left( \|(XX^\top)^{\dagger^{1/2}} \varepsilon \|_{2} \le \frac{2 C' K_\varepsilon \sqrt{n} }{s_{n}(A) \left((1-\alpha) \sqrt{p} - C_{K_{Z}} \sqrt{n}\right)} \Big| Z\right)  1_{Z \in \mathcal E_Z} \right]
\end{align*}
which gives
\begin{align*}
    \mathbb P_{\varepsilon,Z} \left(\mathcal E_{\varepsilon,Z}^c \right) & = 1-\mathbb E_{Z} \left[\mathbb P_{\varepsilon} \left( \|(XX^\top)^{\dagger^{1/2}} \varepsilon \|_{2} \le \frac{2 C' K_\varepsilon \sqrt{n} }{s_{n}(A) \left((1-\alpha) \sqrt{p} - C_{K_{Z}} \sqrt{n}\right)} \Big| Z\right)  1_{Z \in \mathcal E_Z} \right] 
    \\ 
     & = \mathbb E_{Z} \left[1- \mathbb P_{\varepsilon} \left( \|(XX^\top)^{\dagger^{1/2}} \varepsilon \|_{2} \le \frac{2 C' K_\varepsilon \sqrt{n} }{s_{n}(A) \left((1-\alpha) \sqrt{p} - C_{K_{Z}} \sqrt{n}\right)} \Big| Z\right)  1_{Z \in \mathcal E_Z} \right] 
    \\ 
    & = \mathbb E_{Z} \left[1-  \left(1-\mathbb P_{\varepsilon} \left( \|(XX^\top)^{\dagger^{1/2}} \varepsilon \|_{2} \ge \frac{2 C' K_\varepsilon \sqrt{n} }{s_{n}(A) \left((1-\alpha) \sqrt{p} - C_{K_{Z}} \sqrt{n}\right)} \Big| Z\right) \right) 1_{Z \in \mathcal E_Z} \right] \\
    & = \mathbb E_{Z} \left[\mathbb P_{\varepsilon} \left( \|(XX^\top)^{\dagger^{1/2}} \varepsilon \|_{2} \ge \frac{2 C' K_\varepsilon \sqrt{n} }{s_{n}(A) \left((1-\alpha) \sqrt{p} - C_{K_{Z}} \sqrt{n}\right)} \Big| Z\right)  1_{Z \in \mathcal E_Z} \right] + 1- \mathbb E_{Z} \left[1_{Z \in \mathcal E_Z} \right] \\
    & \le \mathbb E_{Z} \left[\exp\left( - c C'^2 n\right) 1_{Z \in \mathcal E_Z} \right] + \mathbb P\left( \mathcal E_Z^c\right) \\
    & \le \exp\left( - c C'^2 n\right) \underbrace{\mathbb P_{Z} \left( Z \in \mathcal E_Z \right)}_{\le 1} + \mathbb P\left( \mathcal E_Z^c\right), 
\end{align*}
which finally yields
\begin{align*}    
     \mathbb P_{\varepsilon,Z} \left(\mathcal E_{\varepsilon,Z}^c \right) & \le \exp\left( - c C'^2 n\right) + 2 \exp \left(-c_{K_Z} \alpha p\right).
\end{align*}


Let us turn to the study of $\vert \langle X^\top (XX^\top)^{\dagger} \varepsilon,X_{n+1}\rangle\vert$. Conditioning on $\varepsilon$, we have  
\begin{align*}
    \mathbb P_{\varepsilon,Z,X_{n+1}}\left( \vert \langle X^\top (XX^\top)^{\dagger} \varepsilon,X_{n+1}\rangle\vert \ge  u\right) 
    & 
    =\mathbb E_{\varepsilon,Z}\left[\mathbb P_{X_{n+1}}\left( \vert \langle X^\top (XX^\top)^{\dagger} \varepsilon,X_{n+1}\rangle\vert \ge u\mid \varepsilon,Z\right)\right] \\
    & 
    =\mathbb E_{\varepsilon,Z}\left[\mathbb P_{X_{n+1}}\left( \vert \langle X^\top (XX^\top)^{\dagger} \varepsilon,X_{n+1}\rangle\vert \ge u\mid \varepsilon,Z\right)I_{\mathcal E_Z \cap \mathcal E_{\varepsilon,Z}}\right] \\
    & \hspace{1cm} + \mathbb E_{\varepsilon,Z}\left[\underbrace{\mathbb P_{X_{n+1}}\left( \vert \langle X^\top (XX^\top)^{\dagger} \varepsilon,X_{n+1}\rangle\vert \ge u\mid \varepsilon,Z\right)}_{\le 1} I_{\mathcal E_Z^c \cup \mathcal E_{\varepsilon,Z}^c}\right] \\
    & =\mathbb E_{\varepsilon,Z}\left[\mathbb P_{X_{n+1}}\left( \vert \langle X^\top (XX^\top)^{\dagger} \varepsilon,X_{n+1}\rangle\vert \ge u\mid \varepsilon,Z\right)I_{\mathcal E_Z \cap \mathcal E_{\varepsilon,Z}}\right] \\
    & \hspace{1cm} + \mathbb P_{\varepsilon,Z}\left(\mathcal E_Z^c \cup \mathcal E_{\varepsilon,Z}^c\right). \notag
\end{align*}

On the other hand, the subGaussianity of $Z_{n+1}$, 
\begin{align*}
    \mathbb{P}_{X_{n+1}} \left( \vert \langle X^\top (XX^\top)^{\dagger} \varepsilon,X_{n+1}\rangle\vert \ge  u\mid \varepsilon,Z\right) & = \mathbb P\left( \vert \langle X^\top (XX^\top)^{\dagger} \varepsilon,A Z_{n+1}\rangle\vert \ge  u\mid \varepsilon,Z\right)\\
    & = \mathbb P\left( \vert \langle A^\top X^\top (XX^\top)^{\dagger} \varepsilon, Z_{n+1}\rangle\vert \ge  u\mid \varepsilon,Z\right)\\
    & \leq 2 \exp \left(-\frac{u^2}{2 \ K_
    Z^2 \ \Vert A^\top X^\top (XX^\top)^{\dagger} \varepsilon\Vert_2^2 } \right) 
\end{align*}
which gives 
\begin{align*}
    \mathbb P_{X_{n+1}} \left( \vert \langle X^\top (XX^\top)^{\dagger} \varepsilon,X_{n+1}\rangle\vert \ge  u\mid \varepsilon,Z\right) & \leq 2 \exp \left(-\frac{u^2}{2 \ K_Z^2 s_{1}(A)^2 \ \Vert X^\top (XX^\top)^{\dagger} \varepsilon\Vert_2^2 } \right).
\end{align*}
Then
\begin{align*}
    \mathbb P_{\varepsilon,Z,X_{n+1}}\left( \vert \langle X^\top (XX^\top)^{\dagger} \varepsilon,X_{n+1}\rangle\vert \ge  u\right)
    &\leq \mathbb E_{\varepsilon,Z}\left[2 \exp \left(-\frac{u^2}{2 \ K_
    Z^2 s_{1}(A)^2 \ \Vert X^\top (XX^\top)^{\dagger} \varepsilon\Vert_2^2 } \right) I_{\mathcal E_Z \cap \mathcal E_{\varepsilon,Z}}\right] \\
    & \hspace{1cm} + \mathbb P_{\varepsilon,Z}\left(\mathcal E_Z^c \cup \mathcal E_{\varepsilon,Z}^c\right) \\
    & \hspace{-3cm}\leq 2 \exp \left(-\frac{u^2}{2 \ K_
    Z^2 s_{1}(A)^2  } \frac{s_{n}(A)^2 \left((1-\alpha) \sqrt{p} - C_{K_{Z}} \sqrt{n}\right)^2}{4 C'^2 K_\varepsilon^2 n} \right)  \mathbb P_{\varepsilon,Z}\left[ \mathcal E_Z \cap \mathcal E_{\varepsilon,Z}\right] \\
    &  + \mathbb P_{\varepsilon,Z}\left(\mathcal E_Z^c \cup \mathcal E_{\varepsilon,Z}^c\right) \\
    & \hspace{-3cm}\leq 2 \exp \left(- \frac{u^2 \left((1-\alpha) \sqrt{p} - C_{K_{Z}} \sqrt{n}\right)^2}{8 \ K_Z^2 \kappa_{n}(A)^2 C'^2 K_\varepsilon^2 n} \right) + \mathbb P_{\varepsilon,Z}\left(\mathcal E_Z^c \cup \mathcal E_{\varepsilon,Z}^c\right), \\
\end{align*}
with $\kappa_n(A) = \frac{s_1(A)}{s_n(A)}$ the $n$-th condition number. 

And since
\begin{align*}
    \mathbb P_{\varepsilon,Z}\left(\mathcal E_Z^c \cup \mathcal E_{\varepsilon,Z}^c\right) &\leq \mathbb P_{\varepsilon,Z}\left(\mathcal E_Z^c) + \mathbb P (\mathcal E_{\varepsilon,Z}^c\right) \leq 2 \exp{\left( -c_{K_Z} \alpha p \right)} + \exp{(- c C'^2 n)} + \mathbb P (\mathcal E_{Z}^c)\\
    &\leq 4 \exp{\left( -c_{K_Z} \alpha p \right)} + \exp\left( - c C'^2 n\right),
\end{align*}
we get
\begin{align*}
    \mathbb P_{\varepsilon,Z,X_{n+1}}\left( \vert \langle X^\top (XX^\top)^{\dagger} \varepsilon,X_{n+1}\rangle\vert \ge  u\right) \leq 2 \exp \left(- \frac{u^2 \left((1-\alpha) \sqrt{p} - C_{K_{Z}} \sqrt{n}\right)^2}{8 \ K_Z^2 C'^2 K_\varepsilon^2 \kappa_{n}(A)^2 n} \right) + 4 \exp{\left( -c_{K_Z} \alpha p \right)} + \exp\left( - c C'^2 n\right).
\end{align*}
Choosing $u = \sqrt{(n/p) \log(n^\tau)}$ for some $\tau>0$, we get
\begin{align}
    \mathbb P_{\varepsilon,Z,X_{n+1}} \left( \vert \langle X^\top (XX^\top)^{\dagger} \varepsilon,X_{n+1}\rangle\vert \ge  \sqrt{(n/p) \log(n^\tau)} \right)
    &\leq 2 \exp \left(- \frac{\tau \log (n) \left(1-\alpha - C_{K_{Z}} \sqrt{n/p} \right)^2}{8 \ K_Z^2 C'^2 K_\varepsilon^2 \kappa_{n}(A)^2} \right) \\
    &\hspace{2cm} + 4 \exp{\left( -c_{K_Z} \alpha p \right)} + \exp\left( - c C'^2 n\right).
    \label{tatitoto4_aniso}
\end{align}

\subsection{Final step of the proof}
Combining \eqref{titutu}, \eqref{predineg_aniso}, \eqref{tatitoto2_aniso}, \eqref{tatitoto4_aniso} and  with \eqref{mainineg}, we get 
 
\begin{align*}
    \vert X_{n+1}^\top (\hat \theta^\circ -\theta^*) \vert & \le \vert X_{n+1}^\top (\hat \theta^\circ -\hat \theta^\sharp) \vert+\vert X_{n+1}^\top (\hat \theta -\theta^*) \vert \\
    &\leq \vert X_{n+1}^\top (\hat \theta^\circ -\hat \theta^\sharp) \vert + t \ \frac{6\sqrt{C} C_{\ell''} C_{f'} K_{\varepsilon} \sqrt{n}}{ \delta ((1-\alpha) \sqrt{p} - C_{K_{Z}} \sqrt{n})} \\
    &\leq \left\vert  X_{n+1}^\top \theta^*\right\vert+ \left\vert X_{n+1}^\top X^\top (XX^\top)^{\dagger} \varepsilon \right\vert+ \left\vert    \varepsilon_{n+1}\right\vert + t \ \frac{6\sqrt{C} C_{\ell''} C_{f'} K_{\varepsilon} \sqrt{n}}{ \delta ((1-\alpha) \sqrt{p} - C_{K_{Z}} \sqrt{n})} \\
    &\leq t \ K_Z \Vert A^\top \theta^*\Vert_2 + \left\vert X_{n+1}^\top X^\top (XX^\top)^{\dagger} \varepsilon \right\vert+ \left\vert    \varepsilon_{n+1}\right\vert + t \ \frac{6\sqrt{C} C_{\ell''} C_{f'} K_{\varepsilon} \sqrt{n}}{ \delta ((1-\alpha) \sqrt{p} - C_{K_{Z}} \sqrt{n})} \\
    &\leq \left\vert \varepsilon_{n+1}\right\vert + t \ K_Z \Vert A^\top \theta^*\Vert_2 + \sqrt{(n/p) \log(n^\tau)}
    + t \ \frac{6\sqrt{C} C_{\ell''} C_{f'} K_{\varepsilon} \sqrt{n}}{ \delta ((1-\alpha) \sqrt{p} - C_{K_{Z}} \sqrt{n})}
\end{align*}
with probability at least 
\begin{multline}
1 -  2 \exp{(-c_{K_{Z}} \alpha^{2} p)}  - \exp{\left( -\frac{n}{2} \right)} \textcolor{black}{- \exp ( -c_{K_{Z}} t^2)} - \exp\left(- \frac{t^2}{2}\right) \\
- 2 \exp \left(- \frac{\tau \log (n) \left(1-\alpha - C_{K_{Z}} \sqrt{n/p} \right)^2}{8 \ K_Z^2 C'^2 K_\varepsilon^2 \kappa_{n}(A)^2} \right) - 4 \exp{\left( -c_{K_Z} \alpha p \right)} - \exp\left(- c C'^2 n\right).\notag 
\end{multline}

Using  $C_{f'}$-Lipschitzianity of $f$, we obtain
\begin{align*}
\vert f(X_{n+1}^\top \hat \theta^\circ)  - f(X_{n+1}^\top\theta^*)\vert
& \le C_{f'} \left( \left\vert \varepsilon_{n+1}\right\vert + t \ K_Z \Vert A^\top \theta^*\Vert_2 + \sqrt{(n/p) \log(n^\tau)}
    + t \ \frac{6\sqrt{C} C_{\ell''} C_{f'} K_{\varepsilon} \sqrt{n}}{ \delta ((1-\alpha) \sqrt{p} - C_{K_{Z}} \sqrt{n})} \right)
\end{align*}
with probability at least 
\begin{multline}
1 -  2 \exp{(-c_{K_{Z}} \alpha^{2} p)}  - \exp{\left( -\frac{n}{2} \right)} \textcolor{black}{- \exp ( -c_{K_{Z}} t^2)} - \exp\left(- \frac{t^2}{2}\right) \\
- 2 \exp \left(- \frac{\tau \log (n) \left(1-\alpha - C_{K_{Z}} \sqrt{n/p} \right)^2}{8 \ K_Z^2 C'^2 K_\varepsilon^2 \kappa_{n}(A)^2} \right) - 4 \exp{\left( -c_{K_Z} \alpha p \right)} - \exp\left(- c C'^2 n\right).\notag 
\end{multline}
which is the desired result.

\section{Classical bounds on the extreme singular values of finite dimensional random matrices}
\subsection{Random matrices with independent rows}

Recall that the matrix $X$ is composed by $n$ i.i.d. subGaussian random vectors in $\mathbb{R}^{p}$, with $K_{X} = \max_{i} \Vert X_{i} \Vert_{\psi_{2}}$.
In the underparametrised case, we have $n > p$. Let us recall  the following bound on the singular values of a matrix with independent subGaussian rows. 

\begin{theo}\cite[Theorem 5.39]{vershynin2010introduction}
    Let $X$ be an $n \times p$ matrix whose rows $X_{i}$, $i=1,\ldots,n$ are independent subGaussian isotropic random vectors in $\mathbb{R}^{p}$. Then for every $t \geq 0$, with probability at least $1-2\exp(-c_{K_{X}} t^2)$, one has
    \[\sqrt{n} - C_{K_{X}} \sqrt{p} - t \leq s_{n}(X) \leq s_{1}(X) \leq \sqrt{n} + C_{K_{X}} \sqrt{p} + t\]
    where $C_{K_{X}}, c_{K_{X}} > 0$ depend only on the subGaussian norm $K_{X} = \max_{i} \Vert X_{i} \Vert_{\psi_{2}}$ of the rows.
    \label{vershrows}
\end{theo}
In the our main text, we use the corollary
\begin{cor}
\label{vershynin_cor}
Let us suppose that $t \geq \alpha \sqrt{n}$ with $\alpha > 0$. Using the same assumptions as in Theorem \ref{vershrows}, then with probability equal or larger than $1-2\exp(-c_{K_{X}} \alpha^2 n)$
\begin{align*}
    & s_{n}(X) \geq (1-\alpha) \sqrt{n} - C_{K_{X}} \sqrt{p}, \\ 
    & s_{1}(X)\leq (1+\alpha) \sqrt{n} + C_{K_{X}} \sqrt{p}.
\end{align*}
\end{cor}

\subsection{Random matrices with independent columns}

In the overparametrised case, the following theorem of Vershynin  will be instrumental. 

\begin{theo}
    \cite[Theorem 5.58]{vershynin2010introduction}
    Let $X$ be an $n \times p$ matrix with $n \le p$ whose rows $X_{i}$ are independent subGaussian isotropic random vectors in $\mathbb{R}^{p}$ with $\Vert X_{i} \Vert_{2} = \sqrt{p}$ a.s. Then for every $t \geq 0$, with probability at least $1-2\exp(-c_{K_{X}} \alpha \ p)$, one has
    \begin{align*}
    & s_{n}(X^\top) \ge  (1-\alpha) \sqrt{p} - C_{K_{X}} \sqrt{n}, \\
    & s_{1}(X^\top)  \le (1+\alpha) \sqrt{p} + C_{K_{X}} \sqrt{n},
    \end{align*}
    where $C_{K_{X}}, c_{K_{X}} > 0$ depend only on the subGaussian norm $K_{X} = \max_{i} \Vert X_{i} \Vert_{\psi_{2}}$ of the columns.
    \label{vershcolumn}
\end{theo}

\end{document}